%% file: ijcai20.tex
\pdfoutput=1
\documentclass{article}
\usepackage{ijcai20}

\usepackage{times}

\usepackage{soul}
\usepackage{url}
\usepackage[utf8]{inputenc}
\usepackage[small]{caption}
\usepackage{graphicx}
\usepackage{amsmath}
\usepackage{booktabs}
\usepackage{algorithm}
\usepackage{algorithmic}
\usepackage{amssymb}
\DeclareMathOperator*{\argmin}{arg\,min}
\DeclareMathOperator*{\argmax}{arg\,max}
\renewcommand{\vec}[1]{\boldsymbol{#1}}

\usepackage{amsthm}
\usepackage{bbm}

\theoremstyle{plain}
\newtheorem{thm}{Theorem}

\newtheorem{prop}[thm]{Proposition}

\theoremstyle{definition}
\newtheorem{defn}{Definition}

\theoremstyle{remark}
\newtheorem*{rem}{Remark}

\usepackage{pgfplots}
\usepackage{subfigure}
\pgfplotsset{width=7cm,compat=1.5}
\usepackage{thmtools,thm-restate}

\title{Evaluating and Aggregating Feature-based Model Explanations}

\author{
Umang Bhatt$^{1,2}$\footnote{Contact Author}\and
Adrian Weller$^{1,3}$\And
Jos\'e M. F. Moura$^2$\\
\affiliations
$^1$University of Cambridge\\
$^2$Carnegie Mellon University\\
$^3$The Alan Turing Institute\\
\emails
\{usb20, aw665\}@cam.ac.uk, moura@ece.cmu.edu
}

\begin{document}

\maketitle

\begin{abstract}
A feature-based model explanation denotes how much each input feature contributes to a model's output for a given data point. As the number of proposed explanation functions grows, we lack quantitative evaluation criteria to help practitioners know when to use which explanation function. This paper proposes quantitative evaluation criteria for feature-based explanations: low sensitivity, high faithfulness, and low complexity. We devise a framework for aggregating explanation functions. We develop a procedure for learning an aggregate explanation function with lower complexity and then derive a new aggregate Shapley value explanation function that minimizes sensitivity.
\end{abstract}

\section{Introduction}
There has been great interest in understanding black-box machine learning models via post-hoc explanations. Much of this work has focused on feature-level importance scores for how much a given input feature contributes to a model's output. These techniques are popular amongst machine learning scientists who want to sanity check a model before deploying it in the real world \cite{Bhatt2019ExplainableML}. 
Many feature-based explanation functions are gradient-based techniques that analyze the gradient flow through a model to determine salient input features \cite{shrikumar2017learning,sundararajan2017axiomatic}. Other explanation functions perturb input values to a reference output and measure the change in the model's output \cite{vstrumbelj2014explaining,shap}.

With many candidate explanation functions, machine learning practitioners find it difficult to pick which explanation function best captures how a model reaches a specific output for a given input. Though there has been work in qualitatively evaluating feature-based explanation functions on human subjects \cite{lage2019human}, there has been little exploration into formalizing quantitative techniques for evaluating model explanations. Recent work has created auxiliary tasks to test if attribution is assigned to relevant inputs \cite{yang2019bim} and has developed tools to verify if the features important to an explanation function are relevant to the model itself \cite{camburu2019i}.

Borrowing from the humanities, we motivate three criteria for assessing a feature-based explanation: sensitivity, faithfulness, and complexity. 
Philosophy of science research has advocated for explanations that vary proportionally with changes in the system being explained \cite{lipton2003inference}; as such, explanation functions should be insensitive to perturbations in the model inputs, especially if the model output does not change.
Capturing relevancy faithfully is helpful in an explanation \cite{ruben2015explaining}. 
Since humans cannot process a lot of information at once, some have argued for minimal model explanations that contain only relevant and representative features \cite{batterman2014minimal}; therefore, an explanations should not be complex (i.e., use few features).

In this paper, we first define these three distinct criteria: low sensitivity, high faithfulness, and low complexity. With many explanation function choices, we then propose methods for learning an aggregate explanation function that combines explanation functions. If we want to find the simplest explanation from a set of explanations, then we can aggregate explanations to minimize the complexity of the resulting explanation. If we want to learn a smoother explanation function that varies slowly as inputs are perturbed, we can leverage an aggregation scheme that learns a less sensitive explanation function. To the best of our knowledge, we are the first to rigorously explore aggregation of various explanations, while placing explanation evaluation on an objective footing. To that end, we highlight the contributions of this paper:
\begin{itemize}
    \item We describe three desirable criteria for feature-based explanation functions: low sensitivity, high faithfulness, and low complexity.
    \item We develop an aggregation framework for combining explanation functions.
    \item We create two techniques that reduce explanation complexity by aggregating explanation functions.
    \item We derive an approximation for Shapley-value explanations by aggregating explanations from a point's nearest neighbors, minimizing explanation sensitivity and resembling how humans reason in medical settings.
\end{itemize}

\section{Preliminaries}
Restricting to supervised classification settings, let $\vec{f}$ be a black box predictor that maps an input $\vec{x} \in \mathbb{R}^d$ to an output $\vec{f}(\vec{x}) \in \mathcal{Y}$. 
An explanation function $\vec{g}$ from a family of explanation functions, $\mathcal{G}$, takes in a predictor $\vec{f}$ and a point of interest $\vec{x}$ and returns importance scores $\vec{g}(\vec{f}, \vec{x}) = \vec{\phi_{\vec{x}}} \in \mathbb{R}^d$ for all features, where $\vec{g}(\vec{f}, \vec{x})_i = \phi_{\vec{x},i}$ (simplified to $\phi_{i}$ in context) is the importance of (or attribution for) feature $x_i$ of $\vec{x}$. By $\vec{g}_j$, we refer to a particular explanation function, usually from a set of explanation functions $\mathcal{G}_m = \{\vec{g}_1,\vec{g}_2, \hdots, \vec{g}_m\}$. 

We denote $D: \mathbb{R}^d \times \mathbb{R}^d  \mapsto \mathbb{R}_{\ge 0}$ to be a distance metric over explanations, while $\rho: \mathbb{R}^d \times \mathbb{R}^d  \mapsto \mathbb{R}_{\ge 0}$ denotes a distance metric over the inputs. 
An evaluation criterion $\mu$  takes in a predictor $\vec{f}$, explanation function $\vec{g}$, and input $\vec{x}$, and outputs a scalar: $\mu(\vec{f}, \vec{g}; \vec{x})$.
$\mathcal{D} = \{(\vec{x}^i, y^i)\}_{i=1}^{n}$ refers to a dataset of input-output pairs, and $\mathcal{D}_x$ denotes all $\vec{x}^i$ in $\mathcal{D}$.

\section{Evaluating Explanations}
\label{sec:eval}
With the number of techniques to develop feature level explanations growing in the explainability literature, picking which explanation function $\vec{g}$ to use can be difficult. 
In order to study the aggregation of explanation functions, we define three desiderata of an explanation function $\vec{g}$.
\label{agg}
\subsection{Desideratum: Low Sensitivity}
We want to ensure that, if inputs are near each other and their model outputs are similar, then their explanations should be close to each other.
Assuming $\vec{f}$ is differentiable, we desire an explanation function $\vec{g}$ to have low sensitivity in the region around a point of interest $\vec{x}$, implying local smoothness of $\vec{g}$. While \cite{melis2018towards} codified the property, \cite{ghorbani2017interpretation} empirically tested explanation function sensitivity. We follow the convention of the former and define \textit{max sensitivity} and \textit{average sensitivity} in the neighborhood of a point of interest $\vec{x}$. 

Let $\mathcal{N}_{r} = \{\vec{z} \in \mathcal{D}_{\vec{x}} \;|\; \rho(\vec{x}, \vec{z}) \leq r,  \vec{f}(\vec{x}) = \vec{f}(\vec{z})\}$ be a neighborhood of datapoints within a radius $r$ of $\vec{x}$. 

\begin{defn}[Max Sensitivity]
Given a predictor $\vec{f}$, an explanation function $\vec{g}$, distance metrics $D$ and $\rho$, a radius $r$, and a point $\vec{x}$,  we define the max sensitivity of $\vec{g}$ at $\vec{x}$ as:
\[
\mu_{\text{M}}(\vec{f},\vec{g}, r; \vec{x}) = \max_{\vec{z} \in \mathcal{N}_{r}} D(\vec{g}(\vec{f}, \vec{x}),\vec{g}(\vec{f},\vec{z}))
\]
\end{defn}

\begin{defn}[Average Sensitivity]
Given a predictor $\vec{f}$, an explanation function $\vec{g}$, distance metrics $D$ and $\rho$, a radius $r$, a distribution $\mathbb{P}_{\vec{x}}(\cdot)$ over the inputs centered at point $\vec{x}$, we define the average sensitivity of $\vec{g}$ at $\vec{x}$ as:
\[
\mu_{\text{A}}(\vec{f},\vec{g}, r; \vec{x}) =\underset{\vec{z} \in \mathcal{N}_{r}}{\int} D(\vec{g}(\vec{f},\vec{x}),\vec{g}(\vec{f},\vec{z}))\mathbb{P}_{\vec{x}}(\vec{z})d\vec{z}
\]
\end{defn}

\subsection{Desideratum: High Faithfulness}
Faithfulness has been defined in \cite{yeh2019sensitive}.
The feature importance scores from $\vec{g}$ should correspond to the important features of $\vec{x}$ for $\vec{f}$; as such, when we set particular features $\vec{x}_s$ to a baseline value $\bar{\vec{x}}_s$, the change in predictor's output should be proportional to the sum of attribution scores of features in $\vec{x}_s$. We measure this as the correlation between the sum of the attributions of $\vec{x}_s$ and the difference in output when setting those features to a reference baseline. For a subset of indices $S \subseteq \{1,2, \ldots d\}$, $\vec{x}_{s} = \{x_i, i \in S\}$ denotes a sub-vector of input features that partitions the input, $\vec{x} = \vec{x}_{s} \cup \vec{x}_{c}$. $\vec{x}_{[\vec{x}_{s} = \bar{\vec{x}}_{s}]}$ denotes an input where $\vec{x}_{s}$ is set to a reference baseline while $\vec{x}_{c}$ remains unchanged: $\vec{x}_{[\vec{x}_{s} = \bar{\vec{x}}_{s}]} = \bar{\vec{x}}_{s} \cup \vec{x}_{c}$. When $|S| = d$, $\vec{x}_{[\vec{x}_{s} = \bar{\vec{x}}_{s}]} = \bar{\vec{x}}$.

\begin{rem}[Reference Baselines]
Recent work has discussed how to pick a proper reference baseline $\bar{\vec{x}}$. \cite{sundararajan2017axiomatic} suggests using a baseline where $\vec{f}(\bar{\vec{x}}) \approx 0$, while others have proposed taking the baseline to be the mean of the training data. \cite{chang2018explaining} notes that the baseline can be learned using generative modeling. 
\end{rem}

\begin{defn}[Faithfulness]
Given a predictor $\vec{f}$, an explanation function $\vec{g}$, a point $\vec{x}$, and a subset size $|S|$, we define the faithfulness of $\vec{g}$ to $\vec{f}$ at $\vec{x}$ as: 
\[
\mu_{\text{F}}(\vec{f},\vec{g}; \vec{x}) = \underset{S \in \binom{[d]}{|S|}}{\text{corr}}\left(\sum_{i \in S}\vec{g}(\vec{f},\vec{x})_i, \vec{f}(\vec{x}) - \vec{f}\left(\vec{x}_{[\vec{x}_{s} = \bar{\vec{x}}_{s}]}\right)\right)
\]
\end{defn}
For our experiments, we fix $|S|$ then randomly sample subsets $\vec{x}_{s}$ of the fixed size from $\vec{x}$ to estimate correlation.
Since we do not see all $\binom{[d]}{|S|}$ subsets in our calculation of faithfulness, we may not get an accurate estimate of the criterion.
Though hard to codify and even harder to aggregate, faithfulness is desirable, as it demonstrates that an explanation captures which features the predictor uses to generate an output for a given input. Learning global feature importances that highlight, in expectation, which features a predictor relies on is a challenging problem left to future work.
\subsection{Desideratum: Low Complexity}
A complex explanation is one that uses all $d$ features in its explanation of which features of $\vec{x}$ are important to $\vec{f}$. Though this explanation may be faithful to the model (as defined above), it may be too difficult for the user to understand (especially if $d$ is large). We define a fractional contribution distribution, where $|\cdot|$ denotes absolute value:
\[
\mathbb{P}_{\vec{g}}(i) = \frac{|\vec{g}(\vec{f},\vec{x})_i|}{\underset{j \in [d]}{\sum} |\vec{g}(\vec{f},\vec{x})_j|}; \; \mathbb{P}_{\vec{g}} = \{\mathbb{P}_{\vec{g}}(1), \ldots, \mathbb{P}_{\vec{g}}(d)\}
\]
Note that $\mathbb{P}_{\vec{g}}$ is a valid probability distribution. Let $\mathbb{P}_{\vec{g}}(i)$ denote the fractional contribution of feature $\vec{x}_{i}$ to the total magnitude of the attribution.
If every feature had equal attribution, the explanation would be complex (even if it is faithful). The simplest explanation would be concentrated on one feature. We define complexity as the entropy of $\mathbb{P}_{\vec{g}}$.

\begin{defn}[Complexity]
Given a predictor $\vec{f}$, explanation function $\vec{g}$, and a point $\vec{x}$, the complexity of $\vec{g}$ at $\vec{x}$ is:
\begin{align*}
\mu_{\text{C}}(\vec{f},\vec{g}; \vec{x}) = \mathbb{E}_i \big [- \ln(\mathbb{P}_{\vec{g}}) \big] & =  - \sum_{i=1}^{d} \mathbb{P}_{\vec{g}}(i) \; \ln(\mathbb{P}_{\vec{g}}(i))
\end{align*}

\end{defn}

\section{Aggregating Explanations}
Given a trained predictor $\vec{f}$, a set of explanation functions $\mathcal{G}_m = \{\vec{g}_1,\ldots, \vec{g}_m\}$, a criterion to optimize $\mu$, and a set of inputs $\mathcal{D}_{\vec{x}}$, we want to find an aggregate explanation function $\vec{g}_{\text{agg}}$ that satisfies $\mu$ at least as well as any $\vec{g}_i \in \mathcal{G}_m$. Let $h(\cdot)$ represent some function that combines $m$ explanations into a consensus $\vec{g}_{\text{agg}} = h(\mathcal{G}_m)$. We now explore different candidates for $h(\cdot)$. 

\subsection{Convex Combination}
\label{lowConv}
Suppose we have two different explanation functions $\vec{g}_{1}$ and $\vec{g}_{2}$ and have chosen a criterion $\mu$ to evaluate a $\vec{g}$. Consider an aggregate explanation, $\vec{g}_{\text{agg}} = h(\vec{g}_{1}, \vec{g}_{2})$. A potential $h(\cdot)$ is a convex combination where $\vec{g}_{\text{agg}} = h(\vec{g}_{1}, \vec{g}_{2}) = w\vec{g}_{1} + (1-w)\vec{g}_{2} = \vec{w}^\intercal\mathcal{G}_m$.
\begin{prop}
\label{conv}
If $D$ is the $\ell_2$ distance and $\mu = \mu_{\text{A}}$ (average sensitivity), the following holds: 
\[\mu_{\text{A}}(\vec{g}_{\text{agg}}) \leq w \mu_{\text{A}}(\vec{g}_{1}) + (1-w) \mu_{\text{A}}(\vec{g}_{2})
\]
\end{prop}
\begin{proof}
Assuming $\mathbb{P}_{\vec{x}}(\vec{z})$ is uniform, we can apply the triangle inequality and the convexity of $D$ to arrive at the above.
\end{proof}

A convex combination of explanation functions thus yields an aggregate explanation function that is at most as sensitive as any of the explanation functions taken alone. In order to learn $w$ given $\vec{g}_{1}$ and $\vec{g}_{2}$, we set up an objective as follows.
\begin{equation}
w^{*} = \underset{w}{\argmin} \underset{\vec{x} \sim \mathcal{D}_{\vec{x}}}{\mathbb{E}} \big [  \mu_{\text{A}}(\vec{g}_{agg}(\vec{f},\vec{x}))\big]
\label{wOPT}
\end{equation}
Assuming a uniform distribution around all $\vec{x} \in \mathcal{D}_{\vec{x}}$, we can rewrite this as:
\begin{equation*}
    w^{*} = \underset{w}{\argmin} \underset{\vec{x} \sim \mathcal{D}_{\vec{x}}}{\int} \; \underset{\vec{z} \in \mathcal{N}_r}{\int}D(\vec{g}_{agg}(\vec{x}), \vec{g}_{agg}(\vec{z}))\mathbb{P}_{\vec{x}}(\vec{z})d\vec{z}d\vec{x}
\end{equation*}
By Cauchy-Schwartz, we get the following:
\[
w^{*} \leq \underset{w}{\argmin} \underset{\vec{x} \sim \mathcal{D}_{\vec{x}}}{\int}\;  \underset{\vec{z} \in \mathcal{N}_r}{\int}D \left(a,b\right)d\vec{z}d\vec{x}
\]

where $a = w\vec{g}_{1}(\vec{f},\vec{x}) + (1-w)\vec{g}_{2}(\vec{f},\vec{x})$ and $b = w\vec{g}_{1}(\vec{f},\vec{z}) + (1-w)\vec{g}_{2}(\vec{f},\vec{z})$. 
This implies that $w^*$ will be minimal when one element of $w^*$ is 0 and the other is 1. Therefore, a convex combination of two explanation functions, found by solving Equation~\eqref{wOPT}, will be at most as sensitive as the least sensitive explanation function.

\subsection{Centroid Aggregation}
\label{meanEXP}
Another sensible candidate for $h(\cdot)$ to combine $m$ explanation functions is based on centroids with respect to some distance function $D: \mathcal{G} \times \mathcal{G} \mapsto \mathbb{R}$, so that:
\[
\vec{g}_{\text{agg}} \in \argmin_{\vec{g} \in \mathcal{G}} \underset{\vec{g}_{i} \in \mathcal{G}_m}{\mathbb{E}} \big [ D(\vec{g},\vec{g}_i)^p \big ] = \argmin_{\vec{g} \in \mathcal{G}} \sum_{i=1}^{m} \, D(\vec{g},\vec{g}_i)^p
\]
where $p$ is a positive constant. The simplest examples of distances are the $\ell_2$ and $\ell_1$ distances with real-valued attributions where $\mathcal{G} \subseteq \mathbb{R}^d$. 
\begin{prop}
\label{MeanProof}
When $D$ is the $\ell_2$ distance and $p = 2$, the aggregate explanation is the feature-wise sample mean.
\begin{equation}
\vec{g}_{\text{agg}}(\vec{f},\vec{x}) =  \vec{g}_{\text{avg}}(\vec{f},\vec{x}) = \frac{1}{m}\sum_{i=1}^m\vec{g}_{i}(\vec{f},\vec{x})
\label{eqn:avg}
\end{equation}
\end{prop}

\begin{prop}
\label{MedProof}
When $D$ is the $\ell_1$ distance and $p = 1$, the aggregate explanation is the feature-wise sample median. \[
\vec{g}_{\text{agg}}(\vec{f},\vec{x}) = \textrm{med}\{\mathcal{G}_m\}
\]
\end{prop}
Propositions \ref{MeanProof} and \ref{MedProof} follow from standard results in statistics that the mean minimizes the sum of squared differences and the median minimizes the sum of absolute deviations \cite{berger2013statistical}.

We could obtain rank-valued attributions by taking any quantitative vector-valued attributions and ranking features according to their values. If $D$ is the Kendall-tau distance with rank-valued attributions where $\mathcal{G} \subseteq \mathcal{S}_d$ (the set of permutations over $d$ features), then the resulting aggregation mechanism via computing the centroid is called the Kemeny-Young rule.  For rank-valued attributions, any aggregation mechanism falls under the rank aggregation problem in social choice theory for which many practical ``voting rules'' exist \cite{Bhatt2019BuildingHT}. 

We analyze the error of a candidate $\vec{g}_\text{agg}$. 
Suppose the optimal explanation for $\vec{x}$ using $\vec{f}$ is $\vec{g}^{*}(\vec{f},\vec{x})$ and suppose $\vec{g}_\text{agg}$ is the mean explanation for $\vec{x}$ in Equation~\eqref{eqn:avg}. Let $\epsilon_{i,\vec{x}} = ||\vec{g}^{*}(\vec{f},\vec{x}) - \vec{g}_{i}(\vec{f},\vec{x})||$ be the error between the optimal explanation and the $i^\text{th}$ explanation function. 

\begin{prop}
\label{error}
The error between the aggregate explanation $\vec{g}_\text{agg}(\vec{f},\vec{x})$ and the optimal explanation $\vec{g}^\text{*}(\vec{f},\vec{x})$ satisfies: $$\epsilon_{\text{agg}} \leq \frac{\sum_{i=1}^{n}\sum_{j=1}^{m}\epsilon_{j,\vec{x}^{i}}}{mn}$$
\end{prop}
\begin{proof}
For a fixed $\vec{x}$, we have:
\begin{align*}
\epsilon_{\text{agg}, \vec{x}} & = ||\vec{g}^{*}(\vec{f},\vec{x}) - \vec{g}_{\text{agg}}(\vec{f},\vec{x})||\\
& = ||\frac{m\vec{g}^{*}(\vec{f},\vec{x})}{m} - \frac{1}{m}\sum_{i=1}^m\vec{g}_{i}(\vec{f},\vec{x})||\\
& \leq \frac{1}{m}\sum_{i=1}^m ||\vec{g}^{*}(\vec{f},\vec{x}) - \vec{g}_{i}(\vec{f},\vec{x})|| = \frac{\sum_{i=1}^m \epsilon_{i, \vec{x}} }{m}
\end{align*}
Averaging across $\mathcal{D}_{\vec{x}}$, we obtain the result.
\end{proof}

Hence, by aggregating, we do better than when using one explanation function alone. Many gradient-based explanation functions fit to noise \cite{2018explainabilityeval}. One way to reduce noise would be to aggregate by ensembling or averaging. As proven in Proposition \ref{error}, the typical error of the aggregate is less than the expected error of each function alone. 

\section{Lowering Complexity Via Aggregation}
\label{sec:comp}
In this section, we describe iterative algorithms for aggregating explanation functions to obtain $\vec{g}_{\text{agg}}(\vec{f}, \vec{x})$ with lower complexity whilst combining $m$ candidate explanation functions $\mathcal{G}_m = \{\vec{g}_1,\ldots, \vec{g}_m\}$. 
We desire a $\vec{g}_{\text{agg}}(\vec{f},\vec{x})$ that contains information from all candidate explanations $\vec{g}_i(\vec{f},\vec{x})$ yet has entropy less than or equal to that of each explanation $\vec{g}_i(\vec{f},\vec{x})$. As discussed, a reasonable candidate for an aggregate explanation function is the sample mean given by Equation~\eqref{eqn:avg}.
We may want $\vec{g}_{\text{agg}}(\vec{f},\vec{x})$ to approach the sample mean, $\vec{g}_{\text{avg}}(\vec{f},\vec{x})$; however, the sample mean may have greater complexity than that of each $\vec{g}_{i}(\vec{f},\vec{x})$.

For example, let $\vec{g}_{1}(\vec{f},\vec{x}) = [-1,0]^T$ and $\vec{g}_{2}(\vec{f},\vec{x}) = [0, 1]^T$. The sample mean is $\vec{g}_{\text{avg}}(\vec{f},\vec{x}) = [-0.5,0.5]^T$. Both  $\vec{g}_1$ and $\vec{g}_2$ have the minimum possible complexity of $0$, while $\vec{g}_{\text{avg}}$ has the maximum possible complexity, $\log(2)$. 
Our aggregation technique must ensure that $\vec{g}_{\text{agg}}(\vec{f},\vec{x})$ approaches $\vec{g}_{\text{avg}}(\vec{f},\vec{x})$ while guaranteeing $\vec{g}_{\text{agg}}(\vec{f},\vec{x})$ has complexity less than or equal to that of each $\vec{g}_i(\vec{f},\vec{x})$.
We now present two approaches for learning a lower complexity explanation, visually represented in Figure \ref{fig:agg}.

\subsection{Gradient-Descent Style Method}
Our first approach is similar to gradient descent. Starting from each $\vec{g}_i(\vec{f},\vec{x})$, we iteratively move towards $\vec{g}_{\text{avg}}(\vec{f},\vec{x})$ in each of the $d$ directions (i.e., changing the $k$th feature by a small amount) if the complexity decreases with that move. We stop moving when the complexity no longer decreases or $\vec{g}_{\text{avg}}(\vec{f},\vec{x})$ is reached.
Simultaneously, we start from $\vec{g}_{\text{avg}}(\vec{f},\vec{x})$ and iteratively move towards each $\vec{g}_{i}(\vec{f},\vec{x})$ in each of the $d$ directions if the complexity decreases. We stop moving when the complexity no longer decreases or any of the $\vec{g}_{i}(\vec{f},\vec{x})$ are reached.
The final $\vec{g}_{\text{agg}}(\vec{f},\vec{x})$ is the location that has the smallest complexity from these $2d$ different walks. 
Since we only move if the complexity decreases and start from each $\vec{g}_i(\vec{f},\vec{x})$, the entropy of $\vec{g}_{\text{agg}}(\vec{f},\vec{x})$ is guaranteed to be less than or equal to the entropy of all $\vec{g}_i(\vec{f},\vec{x})$. 
\subsection{Region Shrinking Method}
In our second approach, we consider the closed region, $\textbf{R}$, which is the convex hull of all the explanation functions, $\vec{g}_i(\vec{f},\vec{x})$. Notice region $\textbf{R}$ initially contains $\vec{g}_{\text{avg}}$. We consider an iterative approach to find the global minimum in the region $\textbf{R}$.
As before, we consider the convex combination formed by two explanation functions, $\vec{g}_i$ and $\vec{g}_j$.
Using convex optimization, we find the value on the line segment between $\vec{g}_i$ and $\vec{g}_j$ that has the minimum complexity; essentially, we iteratively shrink the region. For the region shrinking method, the convex combination formed by $\vec{g}_i$ and $\vec{g}_j$ is:
\begin{equation*}
\label{convex}
 w (\vec{g}_i) + (1-w)(\vec{g}_j), w \in [0,1]
\end{equation*}For every pair of functions in $\mathcal{G}_m$, we find the functions that produces the minimum complexity in the convex combination of the functions, producing a new set of candidates $\mathcal{G}_m^\prime$. $\vec{g}_{\text{agg}}$ is the element in set $\mathcal{G}_m^\prime$ with minimal complexity after $K$ iterations.
In each iteration, a function is chosen if it has the minimum complexity of all the functions in a convex combination. Thus, the minimum complexity of the set $\mathcal{G}_m^\prime$ decreases or remains constant with each iteration.
\begin{figure}
\centering
\subfigure{%
\label{fig:ex3-a}%
\includegraphics[width=0.24\textwidth]{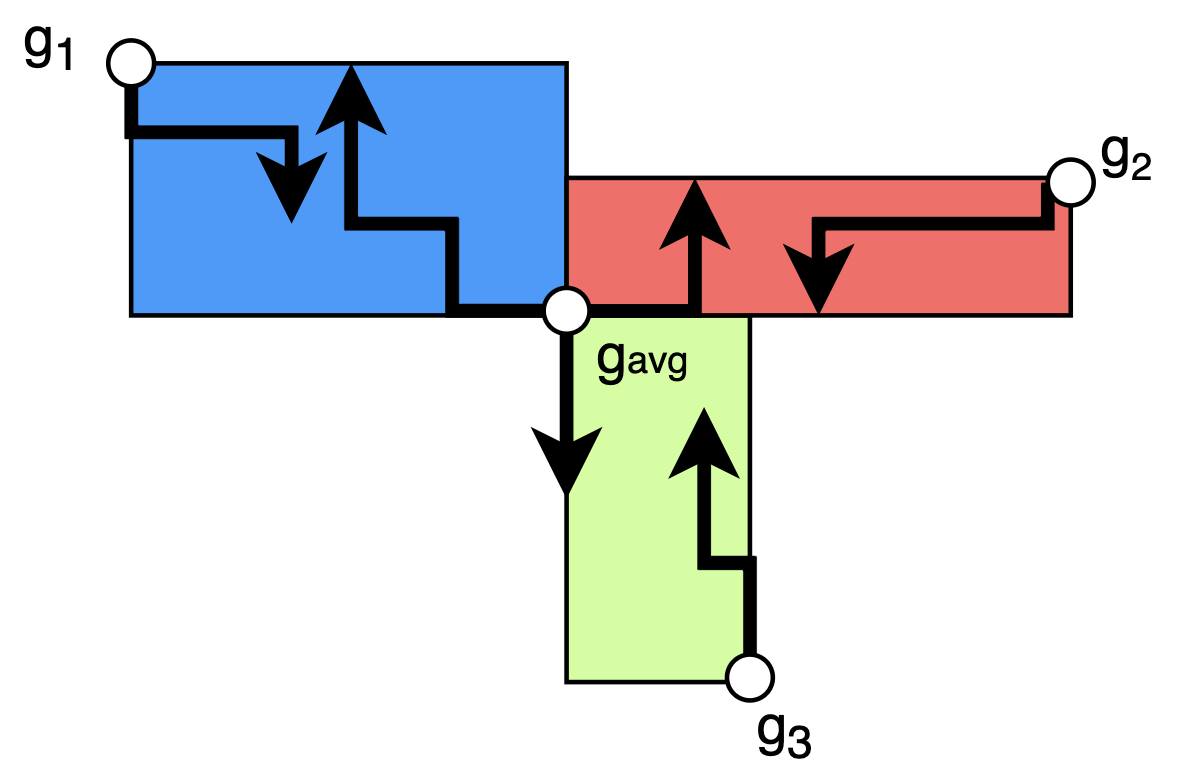}}%
\subfigure{%
\label{fig:ex3-b}%
\includegraphics[width=0.24\textwidth]{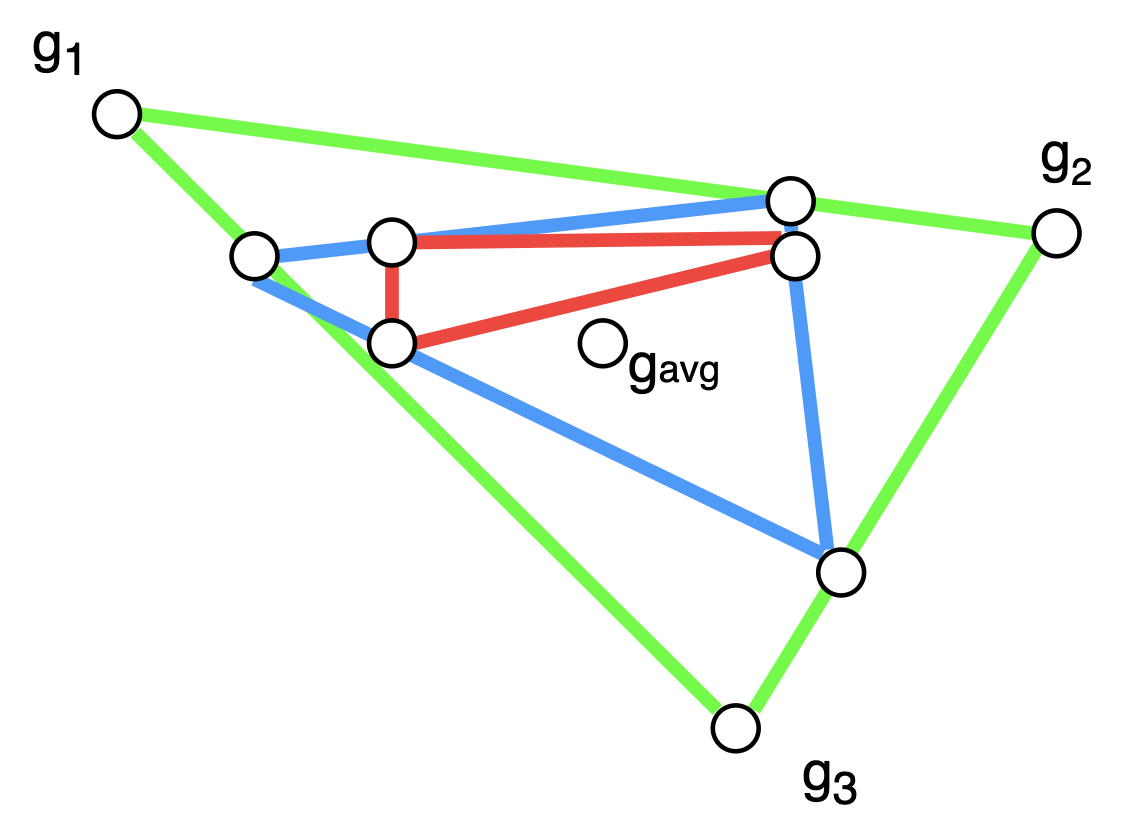}} \\
\caption{Visual examples of the two complexity lowering aggregation algorithms: gradient-descent style \subref{fig:ex3-a} and region shrinking \subref{fig:ex3-b} methods using explanation functions $\vec{g}_1,\vec{g}_2,\vec{g}_3$}
\label{fig:agg}
\end{figure}

\section{Lowering Sensitivity Via Aggregation}
\label{sec:sen}
To construct an aggregate explanation function $\vec{g}$ that minimizes sensitivity, we would need to ensure that a test point's explanation is a function of the explanations of its nearest neighbors under $\rho$. This is a natural analog for how humans reason: we use past similar events (training data) and facts about the present (individual features) to make decisions \cite{bhatt2019towards}. We now contribute a new explanation function $\vec{g}_{\text{AVA}}$ that combines the Shapley value explanations of a test point's nearest neighbors to explain the test point.

\subsection{Shapley Value Review}
Borrowing from game theory, Shapley values denote the marginal contributions of a player to the payoff of a coalitional game. Let $T$ be the number of players and let $v: 2^T \rightarrow \mathbb{R}$ be the characteristic function, where $v(S)$ denotes the worth (contribution) of the players in $S \subseteq T$. The Shapley value of player $i$'s contribution (averaging player $i$'s marginal contributions to all possible subsets $S$) is:
\[
\phi_i(v) = \frac{1}{|T|}\sum_{S \subseteq T \setminus \{i\}} \binom{T - 1}{S}^{-1}(v(S \cup \{i\}) - v(S))
\]
Let $\Phi \in \mathbb{R}^{T}$ be a Shapley value contribution vector for all players in the game, where $\phi_i(v)$ is the $i^{\text{th}}$ element of $\Phi$.

\subsection{Shapley Values as Explanations}
In the feature importance literature, we formulate a similar problem to where the game's payoff is the predictor's output $y = \vec{f}(\vec{x})$, the players are the $d$ features of $\vec{x}$, and the $\phi_i$ values represent the contribution of $x_i$ to the game $\vec{f}(\vec{x})$. Let the characteristic function be the importance score of a subset of features $\vec{x}_s$, where $\mathbb{E}_{Y} [\cdot | \vec{x}]$ is an expectation over $\mathbb{P}_{\vec{f}}(\cdot|\vec{x})$:
\[
v_{\vec{x}}(S) = \mathbb{E}_{Y}\bigg [ - \text{log}\frac{1}{\mathbb{P}_{\vec{f}}(Y|\vec{x}_{s})} \bigg | \vec{x} \bigg]
\]
This characteristic function denotes the negative of the expected number of bits required to encode the predictor's output based on the features in a subset $S$ \cite{chen2018shapley}.
Shapley value contributions can be approximated via Monte Carlo sampling \cite{vstrumbelj2014explaining} or via weighted least squares \cite{shap}.

\subsection{Aggregate Valuation of Antecedents}
We now explore how to explain a test point in terms of the Shapley value explanations of its neighbors. Termed Aggregate Valuation of Antecedents (AVA), we derive an explanation function that explains a data point in terms of the explanations of its neighbors. We do the following: suppose we want to find an explanation function $\vec{g}_{\text{AVA}}(\vec{f}, \vec{x}_{\text{test}})$ for a point of interest $\vec{x}_{\text{test}}$. First we find the $k$ nearest neighbors of $\vec{x}_{\text{test}}$ under $\rho$ denoted by $\mathcal{N}_k(\vec{x}_{\text{test}},\mathcal{D})$. 
\[\mathcal{N}_k(\vec{x}_{\text{test}},\mathcal{D}) = \argmin_{\mathcal{N} \subset \mathcal{D}, |\mathcal{N}| = k} \space \sum_{\vec{z} \in \mathcal{N}} \rho(\vec{x}_{\text{test}}, \vec{z})\]
We define $\vec{g}_{\text{AVA}}(\vec{f},\vec{x}_{\text{test}}) = \Phi_{\vec{x}_{\text{test}}}$ as the explanation function where:
\begin{align*}
\vec{g}_{\text{AVA}}(\vec{f},\vec{x}_{\text{test}})_{i} & = \phi_{i}(v_{\text{AVA}})  =  \sum_{\vec{z} \in \mathcal{N}_k(\vec{x}_{\text{test}})} \frac{\vec{g}_{\text{SHAP}}(\vec{f},\vec{z})_{i}}{\rho(\vec{x}_{\text{test}}, \vec{z})}\\ & = \sum_{\vec{z} \in \mathcal{N}_k(\vec{x}_{\text{test}})} \frac{\phi_i(v_{\vec{z}})}{\rho(\vec{x}_{\text{test}}, \vec{z})}    
\end{align*}

In essence, we weight each neighbor's Shapley value contribution by the inverse distance from the neighbor to the test point. AVA is closely related to bootstrap aggregation from classical statistics, as we take an average of model outputs to improve explanation function stability.

\begin{thm}\label{ava}
$\vec{g}_{\text{AVA}}(\vec{f},\vec{x}_{\text{test}})$ is a Shapley value explanation.
\end{thm}

\begin{proof}
We want to show that $\vec{g}_{\text{AVA}}(\vec{f},\vec{x}_\text{test}) = \Phi_{\vec{x}_\text{test}}$ is indeed a vector of Shapley values.
Let $\vec{g}_{\text{SHAP}}(\vec{f},\vec{z}) = \Phi_{\vec{z}}$ be the vector of Shapley value contributions for a point $\vec{z} \in \mathcal{N}_k$. By \cite{shap}, we know $\vec{g}_{\text{SHAP}}(\vec{f},\vec{z})_{i} = \phi_i(v_{\vec{z}})$ is a unique Shapley value for the characteristic function $v_{\vec{z}}$. By linearity of Shapley values \cite{shapley52},
we know that:
\begin{equation}
    \label{ava_1}
    \phi_i(v_{\vec{z}_{1}} + v_{\vec{z}_{2}}) = \phi_i(v_{\vec{z}_{1}}) + \phi_i(v_{\vec{z}_{2}})
\end{equation}
This means that the $\Phi_{\vec{z}_{1}} + \Phi_{\vec{z}_{2}}$ will yield a unique Shapley value contribution vector for the characteristic function $v_{\vec{z}_{1}} + v_{\vec{z}_{2}}$. By linearity (or additivity), we know for any scalar $\alpha$:
\begin{equation}
\label{ava_2}
    \alpha \phi_i(v_{\vec{z}}) = \phi_i(\alpha v_{\vec{z}})
\end{equation}
This means $\alpha \Phi_{\vec{z}}$ will yield a unique Shapley value contribution vector for the characteristic function $\alpha v_{\vec{z}}$. Now define:
\begin{equation*}
    \Phi_{\vec{x}_\text{test}} = \sum_{\vec{z} \in \mathcal{N}_k(\vec{x}_{\text{test}})} \frac{\Phi_{\vec{z}}}{\rho(\vec{x}_{\text{test}},\vec{z})}
\end{equation*}
We can conclude that $\Phi_{\vec{x}_\text{test}}$ is a vector of Shapley values.
\end{proof}
While \cite{sundararajan2017axiomatic} takes a path integral from a fixed reference baseline $\bar{\vec{x}}$ and \cite{shap} only considers attribution along the straight line path between $\bar{\vec{x}}$ and $\vec{x}_{\text{test}}$, AVA takes a weighted average of attributions along paths from training points in $\mathcal{N}_k$ to $\vec{x}_{\text{test}}$. 
AVA can similarly be thought of as a convex combination of explanation functions where the explanation functions are the explanations of the neighbors of $\vec{x}_\text{test}$ and the weights are $\rho(\vec{x}_{\text{test}}, \vec{z})^{-1}$. Though the weights are guaranteed to be non-negative, we normalize the weights to sum to $1$ and edit the AVA formulation to be: $\vec{g}_{\text{AVA}}(\vec{f},\vec{x}_{\text{test}}) = \rho_{\text{tot}}\Phi_{\vec{x}_{\text{test}}}$ where $\rho_{\text{tot}} = \sum_{\vec{z} \in \mathcal{N}_k(\vec{x}_{\text{test}})} \rho(\vec{x}_{\text{test}}, \vec{z})^{-1}$. Notice this formulation is a specific convex combination as described before; therefore, AVA will result in a lower sensitivity than $\vec{g}_{\text{SHAP}}(\vec{f},\vec{x})$ alone.

\subsection{Medical Connection}
\label{med}
Similar to how a model uses input features to reach an output, medical professionals learn how to proactively search for risk predictors in a patient.
Medical professionals not only use patient attributes (e.g., vital signs, personal information) to make a diagnosis but also leverage experiences with past patients; for example, if a doctor treated a rare disease over a decade ago, then that experience can be crucial when attributes alone are uninformative about how to diagnose \cite{goold1999doctor}. 
This is the analogous to ``close'' training points affecting a predictor's output. 
AVA combines the attributions of past training points (past patients) to explain an unseen test point (current patient). When using the MIMIC dataset \cite{mi}, AVA models the aforementioned intuition.

\begin{table*}
\begin{center}
\begin{small}
\begin{sc}
\begin{tabular}{ccccc}
\toprule
Input & Best (DeepLift) &  Convex & Gradient-Descent & Region-Shrinking \\
\midrule
\begin{minipage}{.12\textwidth}
      \includegraphics[width=\linewidth]{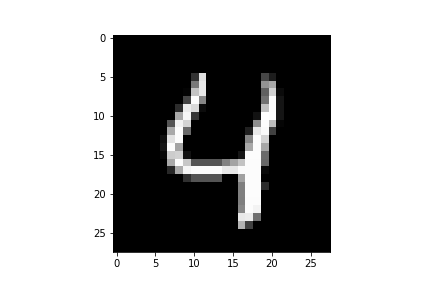}
\end{minipage}
& \begin{minipage}{.12\textwidth}
      \includegraphics[width=\linewidth]{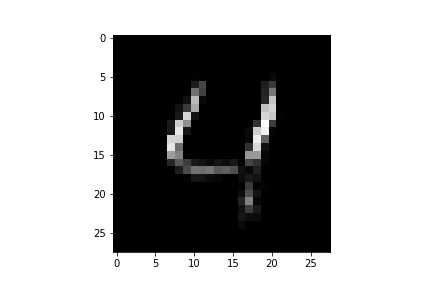}
\end{minipage} 
& \begin{minipage}{.12\textwidth}
      \includegraphics[width=\linewidth]{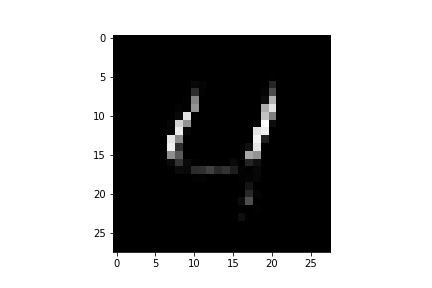}
\end{minipage}  
& \begin{minipage}{.12\textwidth}
      \includegraphics[width=\linewidth]{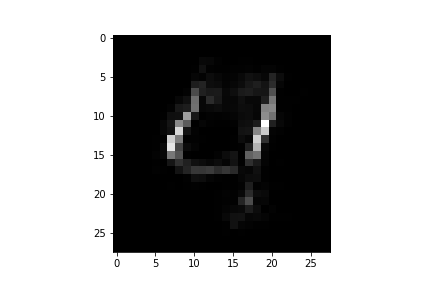}
\end{minipage} 
& \begin{minipage}{.12\textwidth}
      \includegraphics[width=\linewidth]{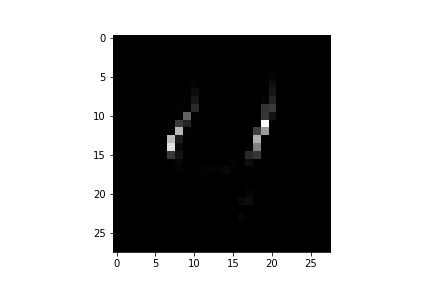}
\end{minipage} \\
  & $\mu_{C} = 3.688$ & $\mu_{C} = 3.685$ & $\mu_{C} = 3.575$ & $\mu_{C} = 3.208$ \\
\bottomrule
\end{tabular}
\end{sc}
\end{small}
\vspace{-0.2cm}
\caption{Qualitative example of aggregation to lower complexity ($\mu_{C}$): We show that it is possible to lower complexity slightly with both of our approaches; note that achieving lowest complexity on an image would imply that all attribution is placed on a single pixel.}
\vspace{-0.1cm}
\label{table:use}
\end{center}
\end{table*}

\section{Experiments}
We now report some empirical results.
We evaluate models trained on the following datasets: Adult, Iris~\cite{Dua:2019}, MIMIC~\cite{mi}, and MNIST~\cite{lecun1998mnist}. We use the following explanation functions: SHAP \cite{shap}, Shapley Sampling (SS)~\cite{vstrumbelj2014explaining}, Gradient Saliency (Grad)~\cite{baehrens2010explain}, Grad*Input (G*I)~\cite{shrikumar2017learning}, Integrated Gradients (IG)~\cite{sundararajan2017axiomatic}, and DeepLift (DL)~\cite{shrikumar2017learning}.

For all tabular datasets, we train a multilayer perceptron (MLP) with leaky-ReLU activation using the ADAM optimizer.  
For Iris \cite{Dua:2019}, we train our model to $96\%$ test accuracy. For Adult \cite{Dua:2019}, our model has $82\%$ test accuracy. As motivated in Section~\ref{med}, we use MIMIC (Medical Information Mart for Intensive Care III) \cite{mi}. 
We extract seventeen real-valued features deemed critical, per \cite{usc}, for sepsis prediction. 
Our model gets $91\%$ test accuracy on the task.
For MNIST \cite{lecun1998mnist}, our model is a convolutional neural network and has $90\%$ test accuracy.

For experiments with a baseline $\bar{\vec{x}}$, zero baseline implies that we set features to $0$ and average baseline uses the average feature value in $\mathcal{D}$. Before doing aggregation, we unit norm all explanations. For the complexity criterion, we take the positive $\ell_1$ norm. We set $D = \ell_2$ and $\rho = \ell_\infty$.

\begin{table}[]
\begin{center}
\begin{small}
\begin{sc}
\begin{tabular}{lcccc}
\toprule
\textbf{Method} & Adult & Iris & MIMIC & MIMIC\\
\textbf{Subset} & 2 & 2 &   10 &  20 \\
\midrule
SHAP & (\textbf{62}, \textbf{60}) & (\textbf{67}, \textbf{68}) &  (31, 36) & (37, 47)\\
SS & (46, 27) & (32, 36) &  (59, \textbf{58}) & (38, 45)\\
Grad  & (30, 53) & (14, 16) & (37, 41) & (28, \textbf{63})\\
G*I & (38, 39) & (27, 30) &  (54, 48) & (59, 43)\\
IG & (47, 33) & (60, 57) &  (66, 51) & (\textbf{68}, 51)\\
DL & (58, 43) & (46, 48) &  (\textbf{84}, 54) & (43, 45)\\
\bottomrule
\end{tabular}
\end{sc}
\end{small}
\end{center}
\caption{Faithfulness $\mu_\text{F}$ averaged over a test set: (Zero Baseline, Training Average Baseline). Exact quantities can be obtained by dividing table entries by $10^{2}$}
\label{tab:faith}
\end{table}

\subsection{Faithfulness $\mu_\text{F}$}
In Table~\ref{tab:faith}, we report results for faithfulness for various explanation functions. When evaluating, we take the average of multiple runs where, in each run, we see at least $50$ datapoints; for each datapoint, we randomly select $|S|$ features and replace them with baseline values. We then calculate the Pearson's correlation coefficient between the predicted logits of each modified test point and the average explanation attribution for only the subset of features. We notice that, as subset size increases, faithfulness increases until the subset is large enough to contain all informative features. We find that Shapley values, approximated with weighted least squares, are the most faithful explanation function for smaller datasets.

\subsection{Max and Avg Sensitivity $\mu_\text{M}$ and $\mu_\text{A}$}
In Table~\ref{tab:sens}, we report the max and average sensitivities for various explanation functions. 
To evaluate the sensitivity criterion, we sample a set of test points from $\mathcal{D}$ and an additional larger set of training points. We then find the training points that fall within a radius $r$ neighborhood of each test point and find the distance between each nearby training point explanation and the test point explanation to get a mean and max. We average over ten random runs of this procedure. Sensitivity is highly dependent on the dimensionality $d$ and on the radius $r$. We find that as sensitivity decreases as $r$ increases. Empirically, for MIMIC, Shapley values approximated by weighted least squares (SHAP) are the least sensitive. 
\begin{table}[]
\label{tab}
\begin{center}
\begin{small}
\begin{sc}
\begin{tabular}{lccc}
\toprule
\textbf{Method} & Adult & Iris & MIMIC \\
\textbf{radius} & 2 & 0.2 &   4 \\
\midrule
SHAP & (60, 54) & (310, 287) & (\textbf{6}, \textbf{5})\\
SS & (191, 168) & (477 , 345) &  (83, 81)\\
Grad & (60, 50) & (\textbf{68}, \textbf{66}) &  (28, 28)\\
G*I & (86, 71) & (298, 279) & (77, 50)\\
IG & (\textbf{19}, \textbf{17}) & (495, 462) & (19, 15)\\
DL & (74, 74) & (850, 820) & (135, 111)\\
\bottomrule
\end{tabular}
\end{sc}
\end{small}
\end{center}
\caption{Sensitivity: (Max $\mu_{\text{M}}$, Avg $\mu_{\text{A}}$). Exact quantities can be obtained by dividing table entries by $10^{3}$}
\label{tab:sens}
\end{table}

\subsection{MNIST Complexity $\mu_\text{C}$}
In Table~\ref{table:use}, we provide a qualitative example for the gradient descent-style and region-shrinking methods for lowering complexity of explanations from a model trained on MNIST. 
We show an example with images since it illustrates the notion of lower complexity well; however, other data types (tabular) might be better suited for complexity optimization. 

\begin{table}[]
\begin{center}
\begin{small}
\begin{sc}
\begin{tabular}{lccc}
\toprule
\textbf{Method} & Adult & Iris &  MIMIC \\
\midrule
$\mu_{A}(\vec{f},\vec{g}_{\text{SHAP}})$ & 0.16  $\pm$  0.11 &  0.22  $\pm$  0.25  &  0.47  $\pm$  0.12  \\
$\mu_{A}(\vec{f},\vec{g}_{\text{AVA}})$ & \textbf{0.07}  $\pm$  \textbf{0.07}  &  \textbf{0.13} $\pm$  \textbf{0.18} &   \textbf{0.31}  $\pm$  \textbf{0.13}  \\
\midrule
$\mu_{M}(\vec{f},\vec{g}_{\text{SHAP}})$ &  0.68  $\pm$  0.13  &  1.20  $\pm$  0.36  &   0.83  $\pm$  0.17  \\
$\mu_{M}(\vec{f},\vec{g}_{\text{AVA}})$ &  \textbf{0.52} $\pm$  \textbf{0.11}  &  \textbf{1.18} $\pm$  \textbf{0.28}  &   \textbf{0.72} $\pm$  \textbf{0.22}\\
\midrule
$\mu_{C}(\vec{f},\vec{g}_{\text{SHAP}})$ & 1.94  $\pm$  0.26  &  1.36  $\pm$  0.36 & \textbf{2.33}  $\pm$  \textbf{0.23}  \\
$\mu_{C}(\vec{f},\vec{g}_{\text{AVA}})$ &  \textbf{1.93} $\pm$  \textbf{0.24}&   \textbf{1.24}  $\pm$  \textbf{0.32}  &   2.61  $\pm$  0.29\\
\bottomrule
\end{tabular}
\end{sc}
\end{small}
\end{center}
\caption{AVA lowers the sensitivity of Shapley value explanations across all datasets. When $d$ is small (fewer features), AVA explanations are slightly less complex.}
\label{table:ava}
\end{table}

\subsection{AVA} 
Our empirical findings support use of an AVA explanation if low sensitivity is desired. 
\cite{ghorbani2017interpretation} note that perturbation-based explanations (like $\vec{g}_{\text{SHAP}}$) are less sensitive than their gradient-based counterparts. 
In Table \ref{table:ava}, we show that AVA explanations not only have lower sensitivities in all experiments but also have less complex explanations (depending on the radius $r$ and number of features $d$). After finding the average distance between pairs of points, we use $r = 1$ for Adult, $r = 0.3$ for Iris, and $r = 10$ for MIMIC. 

\section{Conclusion}
Borrowing from earlier work in social science and the philosophy of science, we codify low sensitivity, high faithfulness, and low complexity as three desirable properties of explanation functions. 
We define these three properties for feature-based explanation functions, develop an aggregation scheme for learning combinations of various explanation functions, and devise schemes to learn explanations with lower complexity (iterative approaches) and lower sensitivity (AVA).
We hope that this work will provide practitioners with a principled way to evaluate feature-based explanations and to learn an explanation which aggregates and optimizes for criteria desired by end users.
Though we consider one criterion at a time, future work could further axiomatize our criteria, explore the interaction between different evaluation criteria, and devise a multi-objective optimization approach to finding a desirable explanation; for example, can we develop a procedure for learning a less sensitive and less complex explanation function simultaneously?

\section*{Acknowledgements}
We thank reviewers for their feedback. We thank Pradeep Ravikumar, John Shi, Brian Davis, Kathleen Ruan, Javier Antoran, James Allingham, and Adithya Raghuraman for their comments and help. 
UB acknowledges support from DeepMind and the Leverhulme Trust via the Leverhulme Center for the Future of Intelligence (CFI) and from the Partnership on AI. AW acknowledges support from the David MacKay Newton Research Fellowship at Darwin College, The Alan Turing Institute under EPSRC grant EP/N510129/1 \& TU/B/000074, and the Leverhulme Trust via the CFI.

\clearpage
\bibliographystyle{named}
\bibliography{ijcai20}

\clearpage
\input{appendix.tex}

\end{document}

%% file: appendix.tex
\appendix
\section{Additional Evaluation Criteria}
In addition to the aforementioned three criteria, there are many other desirable criteria for a $\vec{g}$. To assist practitioners, we now collect and list these additional quantitative evaluation criteria for feature-level explanations. It is possible to evaluate all criteria for both perturbation-based explanations \cite{vstrumbelj2014explaining,shap} and gradient-based explanations \cite{sundararajan2017axiomatic,shrikumar2017learning}. Note we omit evaluation criteria that assume access to ground-truth explanations for training points; for a thorough treatment on this topic, see~\cite{hind2019ted,osman2020ground}. We do not delve into human-centered evaluation of explanation functions either; see~\cite{gilpin2018explaining,poursabzi2018manipulating,yang2019evaluating} for detailed discussions.

\subsubsection{Predictability of Explanations}
We would want to ensure that explanations from $\vec{g}$ are predictable. As such, $\vec{g}(\vec{f},\vec{x})$ ought not vary over function calls. \cite{ecf} notes that identical inputs should give the identical explanations.
\begin{defn}[Identity]
Given a predictor $\vec{f}$, an explanation function $\vec{g}$, and distance metrics $D$ and $\rho$, we define the identity criterion for $\vec{g}$ on $\mathcal{D}$ as:
\begin{align*}
\mu_{\text{IDENTITY}}(\vec{f},\vec{g}) & = & \mathbb{E}_{\vec{x} \in \mathcal{D}_x} \big[ D(\vec{g}(\vec{f},\vec{x}),\vec{g}(\vec{f},\vec{x})) \big ]\\
& = & \mathbb{E}_{\vec{x} \sim \mathcal{D}_x} \big[ ||\vec{g}(\vec{f},\vec{x}) - \vec{g}(\vec{f},\vec{x})||_0 \big ]
\end{align*}
\end{defn}
Note the above two are equivalent and we take the $\ell_0$ norm of the difference between two separate calls to $\vec{g}$ with the same input $\vec{x}$. The identity criterion favors non-stochastic explanation functions. We would want to ensure that any non-identical inputs should have non-identical explanations.
\begin{defn}[Separability]
Given a predictor $\vec{f}$, an explanation function $\vec{g}$, and distance metrics $D$ and $\rho$, we define the separability of $\vec{g}$ on $\mathcal{D}$ as:
\begin{align*}
\mu_{\text{SEP}}(\vec{f},\vec{g}) & = &\mathbb{E}_{\vec{x},\vec{z} \in \mathcal{D}_x, \vec{x} \neq \vec{z}} \big[ D(\vec{g}(\vec{f}, \vec{x}),\vec{g}(\vec{f},\vec{z})) \big ] \\
& = & \mathbb{E}_{\vec{x},\vec{z} \in \mathcal{D}_x, \vec{x} \neq \vec{z}} \big[ ||\vec{g}(\vec{f},\vec{x}) - \vec{g}(\vec{f},\vec{z})||_0 \big ] 
\end{align*}

\end{defn}
We would also want to know how surprising an explanation $\vec{g}(\vec{f},\vec{x})$ is compared to explanations for training data. \cite{hazard2019natively} defines conviction of an input $\vec{x}$ with respect to $\mathcal{D}_{\vec{x}}$ for $k$-Nearest Neighbor algorithms; similarly, we define the conviction of $\vec{g}(\vec{f},\vec{x})$ to explanations of training points, $\mathcal{D}_{\vec{x}}$, using $\vec{g}$.

\begin{defn}[Conviction]
Given a predictor $\vec{f}$, an explanation function $\vec{g}$, a probability distribution over explanations $\mathbb{P}_{\vec{\phi}}(\cdot)$, and a data point $\vec{x}$, we define the conviction of $\vec{g}$ at $\vec{x}$ for $\mathcal{D}$ as:
\[
\mu_{\text{CON}}(\vec{f},\vec{g}, \mathbb{P}_{\vec{\phi}}; \vec{x})  = \frac{\mathbbm{E}_{\vec{z} \sim \mathcal{D}_{\vec{x}}}[I(\vec{g}(\vec{f},\vec{z}) )]}{I(\vec{g}(\vec{f},\vec{x}))}
\]
\indent where $I(\vec{g}(\vec{f},\vec{x})) = - ln(\mathbb{P}_{\vec{\phi}}(\vec{g}(\vec{f},\vec{x})))$
\end{defn}
$\mu_{\text{CON}} = 0$ means that $\vec{g}(\vec{f},\vec{x})$ is surprising. As $\mu_{\text{CON}} \rightarrow \infty$, $\vec{g}(\vec{x})$ contains an expected amount of surprisal and can reasonably occur. We desire a higher $\mu_{\text{CON}}$, implying that $\vec{g}$ behaves predictably.  By changing the distribution to $\mathbb{P}_{\vec{\phi}}(\cdot | y = \vec{f}(\vec{x}))$, the numerator to conditional entropy where $\vec{f}(\vec{z}) = \vec{f}(\vec{x})$, and self-information to $I(\vec{g}(\vec{f},\vec{x})) = -ln(\mathbb{P}_{\vec{\phi}}(\vec{g}(\vec{f},\vec{x})| y = \vec{f}(\vec{x})))$, we define the \textit{conditional conviction} of $\vec{g}(\vec{f}, \vec{x})$ to explanations of the same predicted class.

Other techniques have also argued that $\vec{g}(\vec{f}, \vec{x})$ should recover the output of the original predictor, $\vec{f}(\vec{x})$. Deemed compatibility, this criterion attempts to use $\vec{g}$ as a simple proxy for reproducing the outputs of the complex $\vec{f}$.
\begin{defn}[Compatibility]
Given a predictor $\vec{f}$ and an explanation function $\vec{g}$, we define the completeness of $\vec{g}$ for a dataset $\mathcal{D}$ as:
\[
\mu_{\text{COM}}(\vec{f},\vec{g}) = \frac{1}{N} \sum_{\vec{x} \in \mathcal{D}_x} \left| \left(\sum_{i=1}^{d}\vec{g}(\vec{f}, \vec{x})_{i}\right) - \vec{f}(\vec{x}) \right|
\]
\end{defn}
The closer $\mu_{\text{COM}}$ is to $0$, the more compatible the explanation function is; that is, the explanation function recovers the complex model's outputs well. This criterion is related to the the \textit{completeness} axiom of some explanation functions~\cite{sundararajan2017axiomatic}. An explanation functions can be built to be compatible with the original model (or complete with respect to $\vec{f}$). This is also related to the notion of \textit{post-hoc accuracy} discussed in~\cite{chen2018}.

\subsubsection{Importance of Explanations}
Not only do we want to ensure that the $\vec{g}$ faithfully identifies the most important features, but we also want to understand how well $\vec{f}$ performs when $\vec{x}_{s}$ is unobserved (or set to a baseline $\vec{x}_{s} = \bar{\vec{x}}_{s}$). In particular, we craft $S$ to contain the indices for the $|S|$ features with the highest $\text{abs}(\vec{g}(\vec{f},\vec{x})_i)$.
\[S = \argmax_{S \subset [d], |S| = k} \space \sum_{i \in S} \text{abs}(\vec{g}(\vec{f},\vec{x})_i)\]
Therefore, $\vec{x}_{s}$ is now a sub-vector of the most important features according to a specific $\vec{g}$. As done in \cite{chang2018explaining}, we define a score $s_{\vec{f}}$ for how confidently $\vec{f}$ predicts an output $y$ in terms of log-odds. 
\[
s_{\vec{f}}(y|\vec{x}) = \text{log}(\mathbb{\hat P}_{\vec{f}}(y|\vec{x})) - \text{log}(1 - \mathbb{\hat P}_{\vec{f}}(y|\vec{x}))
\]

\begin{defn}[Deletion]
Given a predictor $\vec{f}$, an explanation function $\vec{g}$, a point of interest $\vec{x}$, a predicted output $y$, and a subset of important features $S$, we define the deletion score for $\vec{f}$ at $\vec{x}$ as:
\[
\mu_{\text{DEL}}(\vec{f},\vec{g};\vec{x},y) = s_{\vec{f}}(y|\vec{x}) - s_{\vec{f}}(y|\vec{x}_{[\vec{x}_{s} = \bar{\vec{x}}_{s}]})
\]
\end{defn}

\begin{defn}[Addition]
Given a predictor $\vec{f}$, an explanation function $\vec{g}$, a point of interest $\vec{x}$, a predicted output $y$, and a subset of important features $S$, we define the addition score for $\vec{f}$ at $\vec{x}$ as:
\[
\mu_{\text{ADD}}(\vec{f},\vec{g};\vec{x},y) =  s_{\vec{f}}(y|\vec{x}_{[\vec{x}_{s} = \bar{\vec{x}}_{s}]}) - s_{\vec{f}}(y|\bar{\vec{x}})
\]
\end{defn}

While the deletion score conveys how the log-odds change when we delete the subset of important features from $\vec{x}$, the addition score tells us how much the log-odds change when we add the subset to the baseline. Instead of re-scoring (via change in log-odds) a modified input like $\vec{x}_{[\vec{x}_{s} = \bar{\vec{x}}_{s}]}$, we can retrain the predictor $\vec{f}$ based on a dataset of modified inputs $\mathcal{D}_{\vec{x}_{[\vec{x}_{s} = \bar{\vec{x}}_{s}]}}$. Addition and Deletion are closely related to \textit{explanation selectivity}, described in~\cite{montavon2018methods}.

Let $\vec{f}_{\bar{\vec{x}}_{s}}$ denote the predictor trained on the modified inputs with the most important pixels removed. As in \cite{2018explainabilityeval}, we define the ROAR score as the difference in accuracy between the original predictor and the modified predictor. We can also train a predictor where the least important features (those in $\vec{x}_c$) are removed. We denote that predictor to be $\vec{f}_{\bar{\vec{x}}_{c}}$ and define a KAR score, as proposed in \cite{2018explainabilityeval}.
\begin{defn}[ROAR]
Given a predictor $\vec{f}$, an explanation function $\vec{g}$, a modified predictor $\vec{f}_{\bar{\vec{x}}_{s}}$, and a subset of important features $S$, we define the ROAR score for $\vec{g}$ on a dataset $\mathcal{D}$ as:
\[
\mu_{\text{ROAR}}(\vec{f},\vec{g}, \vec{f}_{\bar{\vec{x}}_{s}}) = \frac{1}{N} \sum_{\vec{x} \in \mathcal{D}_x} \mathbbm{1}[\vec{f}(\vec{x}) = y] - \mathbbm{1}[\vec{f}_{\bar{\vec{x}}_{s}}(\vec{x}) = y]
\]
\end{defn}

\begin{defn}[KAR]
Given a predictor $\vec{f}$, an explanation function $\vec{g}$, a modified predictor $\vec{f}_{\bar{\vec{x}}_{c}}$, and a subset of important features $S$, we define the KAR score for $\vec{g}$ on a dataset $\mathcal{D}$ as:
\[
\mu_{\text{KAR}}(\vec{f},\vec{g}, \vec{f}_{\bar{\vec{x}}_{c}}) = \frac{1}{N} \sum_{\vec{x} \in \mathcal{D}_x} \mathbbm{1}[\vec{f}(\vec{x}) = y] - \mathbbm{1}[\vec{f}_{\bar{\vec{x}}_{c}}(\vec{x}) = y]
\]
\end{defn}

\subsubsection{Other Connections}
We can also draw parallels between the three criteria proposed in the main paper and existing criteria in the literature.

Low sensitivity is discussed  as \textit{stability} in \cite{melis2018towards}, as \textit{explanation continuity} in \cite{montavon2018methods}, as \textit{sensitivity} in \cite{yeh2019sensitive}, and as \textit{reliability} in \cite{kindermans2019reliability}.

High faithfulness appears as \textit{relevance} in \cite{samek2016evaluating}, as \textit{gold set} in \cite{ribeiro2016should}, as \textit{faithfulness} in \cite{plumb2018model}, as \textit{sensitivity-n} in \cite{ancona2018towards}, and as \textit{infidelity} in~\cite{yeh2019sensitive}.

Low complexity is loosely related to \textit{information gain} from \cite{bylinskii2018different} and to \textit{descriptive sparsity} from \cite{warnecke2019evaluating}.

Moreover, very recent literature has also tried to develop various other explanation function criteria: parameter randomization~\cite{adebayo2018sanity}, clustering-based interpretations~\cite{carter2019made}, existence of ``unexplainable components''~\cite{zhang2019unified}, variants of perturbation techniques~\cite{grabskabarwiska2020measuring}, variants of mutual information measures~\cite{nif}, impact of iterative feature removal~\cite{rieger2020irof},  
and necessity and sufficiency of attributions~\cite{wang2020interpreting}.

\newpage
\section{Proofs}
For thoroughness, we elaborate on the proofs from the main paper here.
\subsection{Proof of Proposition 1}
\begin{proof}
Assuming we fix the predictor $\vec{f}$, let $\vec{g}(\vec{x}) = \vec{g}(\vec{f},\vec{x})$ and let $\int$ represent $\underset{\rho(\vec{x},\vec{z}) \leq R}{\int}$ for the rest of this proof.
\begin{align*}
& \mu_{\text{A}}(\vec{g}_{agg})= \int D(\vec{g}_{agg}(\vec{x}),\vec{g}_{agg}(\vec{z}))\mathbb{P}_{\vec{x}}(\vec{z})d\vec{z}\\
& = \int||\vec{g}_{agg}(\vec{x}) - \vec{g}_{agg}(\vec{z})||_2 d\vec{z}\\
& = \int||w\vec{g}_{1}(\vec{x}) + (1-w)\vec{g}_{2}(\vec{x}) - w\vec{g}_{1}(\vec{z}) - (1-w)\vec{g}_{2}(\vec{z})||_2 d\vec{z}\\
& = \int ||w\vec{g}_{1}(\vec{x}) - w\vec{g}_{1}(\vec{z}) + (1-w)\vec{g}_{2}(\vec{x}) - (1-w)\vec{g}_{2}(\vec{z})||_2 d\vec{z}\\
& = \int||w(\vec{g}_{1}(\vec{x}) - \vec{g}_{1}(\vec{z})) + (1-w)(\vec{g}_{2}(\vec{x}) - \vec{g}_{2}(\vec{z}))||_2 d\vec{z}\\
& \leq \int||w(\vec{g}_{1}(\vec{x}) - \vec{g}_{1}(\vec{z}))||_2 + ||(1-w)(\vec{g}_{2}(\vec{x}) - \vec{g}_{2}(\vec{z}))||_2 d\vec{z}\\
& \leq \int w||\vec{g}_{1}(\vec{x}) - \vec{g}_{1}(\vec{z})||_2 + (1-w)||\vec{g}_{2}(\vec{x}) - \vec{g}_{2}(\vec{z})||_2 d\vec{z}\\
& \leq \int w D(\vec{g}_{1}(\vec{x}),\vec{g}_{1}(\vec{z})) + (1-w)D(\vec{g}_{2}(\vec{x}), \vec{g}_{2}(\vec{z})) d\vec{z}\\
& \leq w\int D(\vec{g}_{1}(\vec{x}),\vec{g}_{1}(\vec{z}))d\vec{z} + (1-w)\int D(\vec{g}_{2}(\vec{x}), \vec{g}_{2}(\vec{z})) d\vec{z}\\
& \leq w \mu_{\text{A}}(\vec{g}_{1}) + (1-w) \mu_{\text{A}}(\vec{g}_{2})
\end{align*}
\end{proof}

\subsection{Proof of Proposition 2}
\begin{proof}
To prove this, we just need to show that the sum of the squared distances is minimized by the mean of a set of explanation vectors:
\[
\vec{g}_{agg}(\vec{f},\vec{x}) = \frac{1}{m}\sum_{i=1}^m\vec{g}_{i}(\vec{f},\vec{x})
\]
Recall we have a set of candidate explanation functions $\mathcal{G}_m = \{\vec{g}_1,\ldots, \vec{g}_m\}$. Fix a point of interest $\vec{x}$. Since $d$ is the $\ell_2$ distance and $p = 2$, we define a loss function as follows:
\[
L(\vec{g}_{\text{agg}}(\vec{f},\vec{x})) = \sum_{i=1}^{m} \, ||\vec{g}_{\text{agg}}(\vec{f},\vec{x})-\vec{g}_{i}(\vec{f},\vec{x})||_2^2
\]
We then compute the partial derivatives with respect to each feature of our aggregate explanation $\vec{g}_{\text{agg}}(\vec{f},\vec{x})_j$.
\begin{gather*}
\frac{\partial L}{\partial \vec{g}_{\text{agg}}(\vec{f},\vec{x})_j} = 2m\vec{g}_{\text{agg}}(\vec{f},\vec{x})_j - 2\sum_{i=1}^{m}\vec{g}_{i}(\vec{f},\vec{x})_j \\
\vec{g}_{\text{agg}}(\vec{f},\vec{x})_j = \frac{\sum_{i=1}^{m}\vec{g}_{i}(\vec{f},\vec{x})_j}{m} \\ 
\vec{g}_{\text{agg}}(\vec{f},\vec{x}) = \begin{bmatrix}
\frac{\sum_{i=1}^{m}\vec{g}_{i}(\vec{f},\vec{x})_1}{m}\\
\vdots\\
\frac{\sum_{i=1}^{m}\vec{g}_{i}(\vec{f},\vec{x})_d}{m}
\end{bmatrix} = \frac{1}{m}\sum_{i=1}^m\vec{g}_{i}(\vec{f},\vec{x})
\end{gather*}
\end{proof}

\subsection{Proof of Proposition 3}
\begin{proof}
To prove this, we just need to show that the sum of the absolute distances is minimized by the median of a set of explanation vectors:
\[
\vec{g}_{agg}(\vec{f},\vec{x}) = \text{med}\{\vec{g}_{i}(\vec{f},\vec{x})\}
\]
Recall we have a set of candidate explanation functions $\mathcal{G}_m = \{\vec{g}_1,\ldots, \vec{g}_m\}$. Fix a point of interest $\vec{x}$. Since $d$ is the $\ell_1$ distance and $p = 1$, we define a loss function as follows:
\[
L(\vec{g}_{\text{agg}}(\vec{f},\vec{x})) = \sum_{i=1}^{m} \, |\vec{g}_{\text{agg}}(\vec{f},\vec{x})-\vec{g}_{i}(\vec{f},\vec{x})|
\]
Taking the partial derivative of the above with respect to each feature of our aggregate explanation $\vec{g}_{\text{agg}}(\vec{x})_j$ yields.
\[
\frac{\partial L}{\partial \vec{g}_{\text{agg}}(\vec{f},\vec{x})_j} = \sum_{i=1}^{m} \text{sign}(\vec{g}_{\text{agg}}(\vec{f},\vec{x})_j -  \vec{g}_{i}(\vec{f},\vec{x})_j) 
\]
Now the above partial derivative only equals zero when the number of positive and negative items are the same. The median is the only value where the number of positive items (those greater than the median) and the number of negative items (those less than the median) are equal. Thus, the median value for each feature $j$ would minimize the sum of absolute deviations loss we crafted above (i.e. $\vec{g}_{\text{agg}}(\vec{f},\vec{x})_j = \text{med}\{\vec{g}_{1}(\vec{f},\vec{x})_j, \vec{g}_{2}(\vec{f},\vec{x})_j, \ldots, \vec{g}_{m}(\vec{f},\vec{x})_j\}$.
\end{proof}

\subsection{Alternative Proof of Theorem 5}
\begin{proof}
We want to show that $\vec{g}_{\text{AVA}}(\vec{f},\vec{x}_\text{test}) = \Phi_{\vec{x}_\text{test}}$ is indeed a vector of Shapley values.
Let $\vec{g}_{\text{SHAP}}(\vec{f},\vec{z}) = \Phi_{\vec{z}}$ be the vector of Shapley value contributions for a point $\vec{z} \in \mathcal{N}_k$. By \cite{shap}, we know that $\vec{g}_{\text{SHAP}}(\vec{f},\vec{z})_{i} = \phi_i(v_{\vec{z}})$ is a unique Shapley value for the characteristic function $v_{\vec{z}}$. By linearity of Shapley values \cite{shapley52},
we know that:
\begin{equation}
    \label{ava_11}
    \phi_i(v_{\vec{z}_{1}} + v_{\vec{z}_{2}}) = \phi_i(v_{\vec{z}_{1}}) + \phi_i(v_{\vec{z}_{2}})
\end{equation}
This means that the $\Phi_{\vec{z}_{1}} + \Phi_{\vec{z}_{2}}$ will yield a unique Shapley value contribution vector for the characteristic function $v_{\vec{z}_{1}} + v_{\vec{z}_{2}}$. By linearity (also called additivity), we also know that, for any scalar $\alpha$:
\begin{equation}
\label{ava_22}
    \alpha \phi_i(v_{\vec{z}}) = \phi_i(\alpha v_{\vec{z}})
\end{equation}
This means that the $\alpha \Phi_{\vec{z}}$ will yield a unique Shapley value contribution vector for the characteristic function $\alpha v_{\vec{z}}$. 
Now, to show $\Phi_{\vec{x}_\text{test}}$ is a vector of Shapley values, it suffices to show that any $\phi_{i}(v_{\text{AVA}}) \in \Phi_{\vec{x}_\text{test}}$ is a Shapley value.
As such, we define $v_{\text{AVA}}$ to be the characteristic function of $\vec{g}_{\text{AVA}}(\vec{f},\vec{x})$, where we find the average weighted importance score of the neighbors of $\vec{x}_\text{test}$.
\begin{align}
\label{ava_3}
v_{\text{AVA}}(S) & =  \sum_{\vec{z} \in \mathcal{N}_k(\vec{x}_{\text{test}})} \frac{v_{\vec{z}}(S)}{\rho(\vec{x}_{\text{test}},\vec{z})}\\
& =  \sum_{\vec{z} \in \mathcal{N}_k(\vec{x}_{\text{test}})} \frac{1}{\rho(\vec{x}_{\text{test}},\vec{z})} \mathbb{E}_{Y}\bigg [ - \text{log}\frac{1}{\mathbb{P}_{\vec{f}}(Y|\vec{z}_{s})} \bigg | \vec{z} \bigg] \nonumber
\end{align}
By Equations \ref{ava_11}, \ref{ava_22}, and \ref{ava_3},  we can see that $\phi_{i}(v_{\text{AVA}})$ is a Shapley value.
\begin{align}
\vec{g}_{\text{AVA}}(\vec{f},\vec{x}_{\text{test}})_{i} & = \phi_{i}(v_{\text{AVA}})\\ & = \sum_{\vec{z} \in \mathcal{N}_k(\vec{x}_{\text{test}})} \frac{\vec{g}_{\text{SHAP}}(\vec{f},\vec{z})_{i}}{\rho(\vec{x}_{\text{test}}, \vec{z})}\nonumber\\ & = \sum_{\vec{z} \in \mathcal{N}_k(\vec{x}_{\text{test}})} \frac{\phi_i(v_{\vec{z}})}{\rho(\vec{x}_{\text{test}}, \vec{z})} \nonumber
\end{align}

\end{proof}

\section{Details on Lowering Complexity}
Given a fixed input $\vec{x}$ and an explanation function $\vec{g}_i$, the complexity can be rewritten as: 
\begin{align*}
\mu_{\text{C}}(\vec{f},\vec{g}_i; \vec{x}) & = - \sum_{k=1}^{d}\frac{|\vec{g}_{i}(\vec{f},\vec{x})_k|}{\underset{j \in [d]}{\sum} |\vec{g}_{i}(\vec{f},\vec{x})_j|}\ln\left( \frac{|\vec{g}_{i}(\vec{f},\vec{x})_k|}{\underset{j \in [d]}{\sum} |\vec{g}_{i}(\vec{f},\vec{x})_j|}\right)
\end{align*}
This will help us determine how a small perturbation of the $k$th component of $\vec{g}_i(\vec{f},\vec{x})$ will affect the complexity of $\vec{g}_i$, which, in turn, will help find a lower complexity explanation.
Note $\vec{g}_{i}(\vec{f},\vec{x})_k$ is the $k$th component of $\vec{g}_i(\vec{f},\vec{x})$. The partial derivative of $\mu_{\text{C}}(\vec{f},\vec{g}_i; \vec{x})$ with respect to the $k$th component of $\vec{g}_{i}(\vec{f},\vec{x})$ is:
\begin{align*}
    \frac{\partial{\mu_{\text{C}}(\vec{f},\vec{g}_i; \vec{x})}}{\partial{\vec{g}_{i}(\vec{f},\vec{x})_k}} = - 
    & 
    \left(1+\ln\left(a\right)\right)\frac{\sum_{\substack{l = 1 \nonumber \\ l\neq k}}^{d} |\vec{g}_{i}(\vec{f},\vec{x})_l|}{(\underset{j \in [d]}{\sum} |\vec{g}_{i}(\vec{f},\vec{x})_j|)^2 }\\ & +
    \sum_{\substack{l = 1 \\ l\neq k}}^{d}  \left(1+\ln\left(b\right)\right)\frac{ |\vec{g}_{i}(\vec{f},\vec{x})_l|}{(\underset{j \in [d]}{\sum} |\vec{g}_{i}(\vec{f},\vec{x})_j|)^2 }
    \label{eqn:1}
\end{align*}
where $a = \frac{|\vec{g}_{i}(\vec{f},\vec{x})_k|}{\underset{j \in [d]}{\sum} |\vec{g}_{i}(\vec{f},\vec{x})_j|}$ and $b = \frac{|\vec{g}_{i}(\vec{f},\vec{x})_l|}{\underset{j \in [d]}{\sum} |\vec{g}_{i}(\vec{f},\vec{x})_j|}$. 

We now provide an additional discussion and comparison of the two algorithms for lowering complexity.

\begin{algorithm}

\caption{Gradient-Descent Style Approach to finding $\vec{g}_{\text{agg}}(\vec{f},\vec{x})$ with lower complexity}
\label{alg1}
\begin{algorithmic} 
\REQUIRE $\alpha$, $\vec{g}_{i}(\vec{f},\vec{x}), i = 1,\hdots,m$, fixed $\vec{x}$

\STATE{$\triangleright$ Calculate the complexity of each $\vec{g}_{i}(\vec{f},\vec{x})$}

\FOR {$i = 1,\hdots,m$}
\STATE $E_{\textbf{g}_i(\textbf{x})} \leftarrow \mu_{\text{C}}(\vec{f},\vec{g}_i; \vec{x}) \leftarrow - \sum_{k=1}^{d}\frac{|\vec{g}_{i}(\vec{f},\vec{x})_k|}{\underset{j \in [d]}{\sum} |\vec{g}_{i}(\vec{f},\vec{x})_j|}\ln\left( \frac{|\vec{g}_{i}(\vec{f},\vec{x})_k|}{\underset{j \in [d]}{\sum} |\vec{g}_{i}(\vec{f},\vec{x})_j|}\right)$
\ENDFOR

\STATE $\vec{g}_{\text{avg}}(\vec{f},\vec{x}) \leftarrow \frac{1}{m}\sum_{i=1}^m\vec{g}_{i}(\vec{f},\vec{x})$

\FOR {$i = 1,\hdots,m$}

\STATE {$\triangleright$ Move in the direction of $\vec{g}_{\text{avg}}(\vec{f},\vec{x})$ from $\vec{g}_{i}(\vec{f},\vec{x})$ as long as the complexity decreases}

\STATE {$\vec{t}_i \leftarrow \vec{g}_i(\vec{f},\vec{x})$}
\WHILE {Complexity of $\vec{t}_i$ is  decreasing and $\vec{t}_i \neq \vec{g}_{\text{avg}}(\vec{f},\vec{x})$ }
\FOR{$j = 1, \hdots,d$}
\STATE {Calculate $\frac{\partial{E_{\vec{t}_i}}}{\partial{\vec{t}_{ij}}}$}

\IF{Complexity decreases by moving in the $j$ direction towards $\vec{g}_{\text{avg}}(\vec{f},\vec{x})$}
\STATE {$\triangleright$ Update $\vec{t}_{ij}$}

\STATE{$\vec{t}_{ij} \leftarrow \vec{t}_{ij} + \alpha\frac{\partial{E_{\vec{t}_{i}}}}{\partial{\vec{t}_{ij}}}$}
\ENDIF
\ENDFOR
\ENDWHILE

\STATE {$\triangleright$ Move in the direction of $\vec{g}_{i}(\vec{f},\vec{x})$ from $\vec{g}_{\text{avg}}(\vec{f},\vec{x})$ as long as the complexity decreases}

\STATE {$\vec{q}_i \leftarrow \vec{g}_{\text{avg}}(\vec{f},\vec{x})$}
\WHILE {Complexity of $\vec{q}_i$ is  decreasing and $\vec{q}_i \neq \vec{g}_{i}(\vec{x})$ }
\FOR{$j = 1, \hdots,d$}
\STATE {Calculate $    \frac{\partial{E_{\vec{q}_i}}}{\partial{\vec{q}_{ij}}}$}

\IF{Complexity decreases by moving in the $j$ direction towards $\vec{g}_{i}(\vec{x})$}
\STATE {$\triangleright$ Update $\vec{q}_{ij}$}

\STATE{$\vec{q}_{ij} \leftarrow \vec{q}_{ij} + \alpha \frac{\partial{E_{\vec{q}_{i}}}}{\partial{\vec{q}_{ij}}}$}

\ENDIF
\ENDFOR
\ENDWHILE
\STATE{$\triangleright$ Take the $\vec{t}_i,\vec{q}_{i}$ that minimizes the complexity}
\STATE {$\vec{b}_{i} = \underset{\vec{x} = \{\vec{q}_{i},\vec{t}_{i}\}}{\text{min}} \ E_{\vec{x}}$}
\ENDFOR

\STATE{$\triangleright$ Take the $\vec{b}_i$ that minimizes the complexity}
\STATE{$\vec{g}_{\text{agg}}(\vec{f},\vec{x})  =  \underset{\vec{b}_{i}}{\text{min}} \ E_{\vec{b}_i}$}
\end{algorithmic}
\end{algorithm}

\begin{algorithm}
\caption{Region Shrinking Approach to finding $\vec{g}_{\text{agg}}(\vec{f},\vec{x})$ with lower complexity}
\label{alg2}
\begin{algorithmic} 
\REQUIRE {$\vec{g}_{i}(\vec{f},\vec{x}), i = 1,\hdots,m$}, fixed $\vec{x}$
\STATE{$t\leftarrow 0$}
\STATE{$\triangleright$ Add all the $\vec{g}_i$ into set $S$}
\STATE{$S \leftarrow \{\vec{g}_{i}(\vec{f},\vec{x}), i = 1,\hdots,m\}$}
\REPEAT 
\STATE{$\triangleright$ Repeat K times}
\STATE{$\triangleright$ Initialize $S'$}
\STATE{$S' \leftarrow \emptyset$}
\FOR{every 2 points in S: $P_1, P_2$}
\STATE{Find point $P$ with the minimum entropy in the convex combination of $P_1, P_2$}
\STATE{Add point $P$ to $S'$}
\ENDFOR
\STATE{$\triangleright$ Update values}
\STATE{Choose the $N$ minimum entropy points in $S'$ to form S}
\STATE{$t\leftarrow t + 1$}
\UNTIL{$t = K$}
\STATE{$\triangleright$ Take the element in set $S$ that minimizes the entropy}
\STATE{$\vec{g}_{\text{agg}}(\vec{f},\vec{x})  =  \underset{k \in S}{\text{min}} \ E_{k}$}
\end{algorithmic}
\end{algorithm}

We  presented two algorithms for finding a $\vec{g}_{\text{agg}}$ with lower complexity: a gradient descent approach (Algorithm \ref{alg1}) and a region shrinking approach (Algorithm \ref{alg2}). 
Algorithm \ref{alg1} relies on a greedy choice of selecting one of the $j$ directions to move in. This algorithm works best for regions that are smooth and with decreasing complexity around $\vec{g}_{i}$ and $\vec{g}_{\text{avg}}$. Since Algorithm \ref{alg1} does not backtrack and moves component-wise, it can avoid areas of higher complexity, but can take a sub-optimal step. For example, consider when $d = 2$. During a walk, Algorithm \ref{alg1} may start at $\vec{g}_{i}$, move in the $x$ direction, but then get stuck as complexity in $y$ direction increases. However, had we moved in the $y$ direction first and then in the $x$ direction, then we may have found a minimum. The choice of component plagues this approach. 
On the other hand, Algorithm \ref{alg2} solves the issue of getting stuck because of regions of high complexity present in Algorithm \ref{alg1}. Since Algorithm \ref{alg2} shrinks the region by choosing points in the convex combination, it a can avoid the areas of high complexity. Since Algorithm \ref{alg2} uses the line segments between the points chosen, it may be difficult to obtain the global minima, which may not occur on the line segment. 
A combination of the two approaches can be used. First, Algorithm \ref{alg2} can be used to shrink the region being considered into a set, $S$, of points with low complexity. This can avoid getting stuck in areas of high complexity, like in Algorithm \ref{alg1}. Then, Algorithm \ref{alg1} can be used to move around the points in set $S$ in order to find the global minima that may not occur on the line segments considered in Algorithm \ref{alg2}. It can refine the points in set $S$ to obtain a lower complexity. In sum, we can shrink the region considered into several candidate sets and then refine the points in each set by perturbing and performing greedy walks around them to find $\vec{g}_{\text{agg}}$ with a low complexity.

\section{Experimental Setup}
We provide additional details on the datasets used and their respective models from our experiments.
\begin{itemize}
    \item Iris \cite{Dua:2019}: The iris dataset consists of 150 datapoints: 50 per class and 4 features per datapoint. We use a one layer multilayer perceptron trained to $96\%$ accuracy as our $\vec{f}$.
    \item Adult \cite{Dua:2019}: Each of the 48842 datapoints has 38 features and falls in one of two classes. Note we label encode categorical attributes. We use a one layer MLP (40 hidden nodes with leaky-relu activation) trained to an accuracy of $82\%$.
    \item Mimic-III \cite{mi}: The MIMIC-III (Medical Information Mart for Intensive Care III) is a large electronic health record dataset compromised of health related data of over 40,000 patients who were admitted to the the critical care units of Beth Israel Deaconess Medical Center between the years 2001 and 2012. MIMIC-III consists of demographics, vital sign measurements, lab test results, medications, procedures, caregiver notes, imaging reports, and mortality of the ICU patients. Using MIMIC-III dataset, we extracted seventeen real-valued features deemed critical in the sepsis diagnosis task as per \cite{usc}. These are the processed features we extracted for every sepsis diagnosis (a binary variable indicating the presence of sepsis): Glasgow Coma Scale, Systolic Blood Pressure, Heart Rate, Body Temperature, Pao2 / Fio2 ratio, Urine Output, Serum Urea Nitrogen Level, White Blood Cells Count, Serum Bicarbonate Level, Sodium Level, Potassium Level, Bilirubin Level, Age, Acquired Immunodeficiency Syndrome, Hematologic Malignancy, Metastatic Cancer, Admission Type. We used two layers of 16 hidden nodes each and leaky-relu activation to get an accuracy of $91\%$ on the sepsis prediction task.
    \item MNIST with CNN \cite{lecun1998mnist}: We use a CNN trained to $90\%$ accuracy with the following architecture: one layer with $32$ $5 \times 5$ filters and ReLU activation; max pooling layer with a $2 \times 2$ filter and stride of $2$; convolutional layer with $64$ $5 \times 5$ filters and ReLU activation; max pooling layer with a $2 \times 2$ filter and stride of 2; final dense layer with 10 output neurons.  We used the MNIST dataset with 60,000 28x28 grayscale images of the 10 digits, along with a test set of 10,000 images.
\end{itemize}
Note that we fix a dataset-model pairing for all experiments.
In practice, when calculating \textit{average sensitivity}, we use the following formulation:
\[
\mu_{\text{A}}(\vec{f},\vec{g}, \vec{x}) = \frac{1}{|\mathcal{N}_r|}\sum_{\vec{z} \in \mathcal{N}_{r}} \frac{D(\vec{g}(\vec{f}, \vec{x}),\vec{g}(\vec{f}, \vec{z}))}{\rho(\vec{x}, \vec{z})}
\]
Effectively, we want to ensure that the distance between an explanation of $\vec{x}$ and an explanation of $\vec{z}$, a point in the neighborhood of $\vec{x}$, is proportional to the distance between $\vec{x}$ and $\vec{z}$. Some recent work has shown that \textit{average sensitivity} can be lowered with simple smoothing tricks to explanation functions or with adversarial training of the predictor itself.